\documentclass[accepted]{uai2026} 
                        
\usepackage[american]{babel}

\usepackage{subcaption}
\usepackage{amsmath}
\usepackage{amssymb}
\usepackage{mathtools}
\usepackage{amsthm}
\usepackage{url}
\usepackage{tikz}
\usepackage{ifthen,amssymb}
\usepackage{mathrsfs}
\usepackage{amsfonts,mathtools}
\usepackage{enumerate}
\usepackage{amsthm}
\usepackage{booktabs} 
\usepackage{multirow}
\usepackage{mathtools}
\usepackage{amsmath}
\usepackage{cases}
\usepackage{appendix}
\usepackage{hyperref}
\usepackage{color}
\usepackage{bm}
\usepackage{bbm}
\usepackage{xcolor}
\usepackage{cancel}
\usepackage{microtype}
\usepackage{wrapfig}
\usepackage{enumitem}
\usepackage{verbatim}
\usepackage{graphicx}
\usepackage{makecell}
\usepackage{array}
\usepackage{tikz}
\usetikzlibrary{arrows.meta}
\usetikzlibrary{fit,backgrounds}

\newcommand{\V}[0]{\mathcal{V}}

\newcommand{\R}[0]{\mathbb{R}}

\newcommand{\Kc}[0]{\mathcal{K}_c}
\newcommand{\Ke}[0]{\mathcal{K}_e}
\newcommand{\Z}[0]{\mathcal{Z}}
\newcommand{\LL}[0]{\mathcal{L}}

\newcolumntype{Y}{>{\raggedright\arraybackslash}X}

\newcommand{\vZc}{\textit{probe}_{Z_c}}
\newcommand{\vZe}{\textit{probe}_{Z_e}}

\newcommand{\xsel}[1][t]{g_{\text{sel}}}

\newcommand{\Trsoft}{\mathcal{T}}

\usepackage[capitalize,noabbrev]{cleveref}

\theoremstyle{plain}
\newtheorem{theorem}{Theorem}[section]

\theoremstyle{definition}

\theoremstyle{remark}
\newtheorem{remark}[theorem]{Remark}

\usepackage{tabularx}
\definecolor{lightblue}{rgb}{0.5, 0.7, 0.95}

\usepackage{natbib} 
\usepackage{mathtools} 
\usepackage{booktabs} 
\usepackage{tikz} 

\title{Inference Time Causal Probing in LLMs}

\author[1]{\href{mailto:<sadegh.khorasani@epfl.ch>?Subject=Your UAI 2025 paper}{Sadegh Khorasani}{}} 
\author[2]{\href{mailto:<s.salehkaleybar@liacs.leidenuniv.nl>?Subject=Your UAI 2025 paper}{Saber Salehkaleybar}{}}
\author[3]{\href{mailto:<negar.kiyavash@epfl.ch>?Subject=Your UAI 2025 paper}{Negar Kiyavash}{}}
\author[1]{\href{mailto:<matthias.grossglauser@epfl.ch>?Subject=Your UAI 2025 paper}{Matthias Grossglauser}{}}

\affil[1]{%
    School of Computer and Communication Sciences, EPFL,\\ Lausanne, Switzerland
}

\affil[2]{%
    Leiden Institute of Advanced Computer Science (LIACS), Leiden University, \\Leiden, The Netherlands
  }

\affil[3]{%
    College of Management of Technology, EPFL, Lausanne, Switzerland
}

\begin{document}
\maketitle

\begin{abstract}
Causal probing methods aim to test and control how internal representations influence the behavior of generative models. In causal probing, an intervention modifies hidden states so that a property takes on a different value. Most existing approaches define such interventions by training an auxiliary probe classifier, which ties the method to a specific task or model and risks misalignment with the model’s predictive geometry. We propose Hidden-state Driven Margin Intervention (HDMI), a probe-free, gradient-based technique that directly steers hidden states using the model's native output. HDMI applies a margin objective that increases the probability of a target continuation while decreasing that of the source, without relying on probe classifiers. We further introduce a lookahead variant (LA-HDMI) for text editing that backpropagates through the softmax embeddings, modifying the current hidden state so that the likelihood of user-specified tokens increases in next token generations while preserving fluency. To evaluate interventions, we measure completeness (whether the targeted property changes as intended) and selectivity (whether unrelated properties are preserved), and report their harmonic mean as an overall measure of reliability. HDMI consistently achieves higher reliability than prior methods on the LGD agreement corpus and the CausalGym benchmark, across Meta-Llama-3-8B-Instruct, and Pythia-70M.

\end{abstract}

\section{Introduction}
\label{sec:intro}
In the study of generative models, a key goal is to understand what latent properties they encode and how these properties shape generation. By property, we mean a linguistic feature of the input that the model may represent internally and that can take on distinct values. For instance, consider the sentence: ``\emph{The cats run across the yard}". Here, the subject has the property of plural number. If we change to ``\emph{The cat runs across the yard}", the property would instead be singular number.

A standard diagnostic approach, \textbf{correlational probing},  trains lightweight classifiers (commonly called “probes”) on hidden states to decode a property (such as sentiment, syntactic role, or topic). These probes reveal what value of property is present in hidden states, but they do not show whether/how the model actually uses this property in predicting the next words. For example, the probe might read off ``plural number" but this does not establish whether plurality is what drives the model to predict ``run" instead of ``runs" for the next token.


\textbf{Causal probing} addresses the aforementioned limitation by \emph{interventions} on the hidden state and tests whether such perturbations alter the next-token distribution of the model. For instance, if we modify the hidden state so that the subject property changes from singular to plural, we would like to see whether the model accordingly modifies the verb from singular to plural form.  Therefore, causal probing investigates not only what information is encoded, but also how it is used in prediction \citep{elazar2021amnesic, ravfogel2020null, kumar2022probing}.
\footnote{Following prior causal probing work \citep{canby2024reliable,ravfogel2021counterfactual,davies2023competence}, we sometimes refer to altered property as ``counterfactuals". However, in the structural causal model framework \citet{pearl2009causality}, counterfactuals require re-evaluating a system under the same randomness, whereas our setting just replaces a hidden representation with a modified one. Thus, our experiments are more accurately considered as interventional queries rather than counterfactual ones.}

In practice, many causal probing methods rely on training a probe to obtain a direction along which to perturb the hidden state. For example, PGD \citep{madry2018towards} trains a probe to classify subject number from hidden states, and then uses the probe’s gradient to adjust the hidden state such that the perturbed state is classified by the probe with the opposite label (e.g., flipping the label from singular to plural).
However, this reliance on probes introduces extra property-specific supervision and training costs, since a separate probe must be trained for each property of interest. Moreover, there is a risk of misalignment in the sense that the probe imposes its own classification boundary on the hidden state which may not coincide with how the model internally encodes and uses a property for generating next tokens.


This motivates our framing of \textbf{inference-time causal probing}. By this, we mean interventions applied entirely at inference time, requiring no probes or retraining of the generative model. Our proposed method, \textbf{Hidden-state Driven Margin Intervention (HDMI)}, directly leverages the generative model’s output head as a native readout. HDMI performs a lightweight gradient-based update to hidden states that shifts the model’s output distribution from the source continuation (e.g., ``runs" for a singular subject) to the interventional target (e.g., ``run” for a plural subject). Concretely, it uses the gradient of the logit margin (defined as the difference between the target and source token logits at the next decoding step), increasing the probability of the target being selected.

The contributions of the paper are as follows:
\begin{itemize}
\item Unlike prior approaches \citep{davies2023competence, ravfogel2021counterfactual,madry2018towards,goodfellow2015explaining} that rely on probes, HDMI directly leverages the generative model head to obtain intervention directions, ensuring alignment with the model’s own predictive geometry. The objective of HDMI is to maximize the logit margin that increases the target continuation’s logit while decreasing the source’s. This formulation naturally extends to cases where a property can be expressed by multiple acceptable continuations (e.g., plural verbs ``are/were" versus singular verbs ``is/was").
Moreover, the gradient computation of the logit margin reduces to a closed-form matrix–vector product, making HDMI computationally lightweight and easy to deploy (Section \ref{sec:HDMI-loss}). 
\item We propose a ``lookahead" variant of HDMI for text editing, where a user provides an edited version for a given input, and the goal is to perturb hidden states to steer generation toward the edited version while preserving fluency.
Lookahead HDMI (LA-HDMI) introduces a softmax embedding transition in the token generation process and modifies the hidden state of every decoding step influenced by future edit positions (Section \ref{sec:HDMI6}).

\item Using the evaluation framework of \citet{canby2024reliable}, compared to previous work, we show that HDMI achieves strong performance across LGD agreement and CausalGym suites. We report completeness and selectivity, metrics that show that interventions alter target properties but avoid altering unrelated properties (Section \ref{sec:exp-setup}).
\end{itemize}

\section{Related Work}
\label{sec:related}

Our work is related to three research directions: (i) probing methods for analyzing the linguistic information encoded in hidden states; (ii) causal interventions on internal mechanisms in language models; and (iii) controllable generation and editing methods that steer model behavior at inference time. 

\paragraph{Probing.}
A long line of work uses lightweight classifiers—“probes”—to read out linguistic properties from intermediate representations \citep{belinkov2022probing, alain2016understanding,hewitt2019designing,pimentel2020information}. Although such probes revealed that models encode rich structure, their interpretability and causal effects on the model predictions have been debated.
The fact that a property can be predicted from embedding representations (or correlated) does not imply the model relies on that property. Some methods have proposed controls to check that a probe's accuracy is not due to dataset quirks or the probe memorizing labels \citep{hewitt2019designing,pimentel2020information}. 
This motivated causal probing, which intervenes on representations and measures behavioral impact.

\paragraph{From correlation to causation.}
Recent work shifts from correlational probing to causal analysis that intervenes on internal representations
\citep{elazar2021amnesic,tucker2021if,ravfogel2021counterfactual,ravfogel2022linear, davies2023competence}. 
Concept–erasure methods attempt to remove information about attributes from representations, often to debias \citep{elazar2018adversarial} or perform causal analysis. Iterative Nullspace Projection (INLP) iteratively removes linearly predictive subspaces for a given attribute from representations and projects representations onto the nullspace of linear predictors for a target attribute \citep{ravfogel2020null,ravfogel2022linear}. ``Amnesic probing'' couples such an erasure with behavioral checks to study whether removing present information about a property changes the model predictions \citep{elazar2021amnesic}. Counterfactual interventions instead modify the hidden state so that the property changes from its factual value to a counterfactual value. Linear counterfactuals such as AlterRep push representations across rowspace hyperplanes to the counterfactual side \citep{ravfogel2021counterfactual}. Nonlinear counterfactuals include gradient‑based interventions (GBIs) that optimize against an attribute probe with FGSM/PGD‑style updates \citep{goodfellow2015explaining,madry2018towards}. 
However, concerns persist that interventions may incompletely transform the target property or inadvertently alter non‑target properties \citep{kumar2022probing, canby2024reliable}. 
In contrast, our approach does not learn a probe for the intervention and instead exploits a direct, model‑native gradient signal from the language model head, aligning the intervention with the model’s predictive geometry at the specific decoding step. 


\paragraph{Controllable model behavior.}
Plug‑and‑Play Language Models (PPLM) perform gradient‑based updates to hidden states to guide the token generations (e.g., change the text topic) using external attribute classifiers, alongside KL regularization to remain on‑manifold \citep{dathathri2019plug}. GeDi \citep{krause2020gedi} and DExperts \citep{liu2021dexperts} provide alternative guidance via discriminators or expert/anti‑expert mixtures to make them more controllable. More recently, representation/activation engineering has explored linear directions to change the behaviors of models \citep{turner2023activation}. 
{
We do not require training an external attribute model or probe. HDMI uses a targeted, per-instance, per-step gradient to change a specified continuation.

\paragraph{Positioning and relation to other literatures.}
HDMI relates to several adjacent literatures but differs in goal and methods. (i) Causal mediation analysis, activation patching, and attribution patching intervene by patching/replacing intermediate activations from a source forward pass into another (e.g., clean/corrupted or factual/counterfactual prompts) to {localize} or {attribute} which internal components (layers/heads/MLPs/positions) mediate an effect \citep{vig2020investigating, sharma2025llms, feucht2025dual, sankaranarayanan2025disjoint}. Attribution patching provides a fast, linearized approximation by using two forward passes plus one backward pass to assign attribution scores to edges/components for the same objective \citep{syed2024attribution, ahsan2025elucidating}. 
In contrast, HDMI assumes that the intervention site is predefined, and its goal is to locally modify a targeted linguistic property encoded in the hidden state. Moreover, it computes a per-instance intervention direction at that decoding step instead of transferring activations across two forward passes.
(ii) {Model editing/representation fine-tuning} methods (e.g., ROME, ReFT) change model parameters to make {persistent} behavioral edits \citep{meng2022locating,wu2404reft}, whereas HDMI/LA-HDMI leave parameters unchanged and apply interventions only during the targeted decoding steps.
}


\section{Problem setting and notations} 
\label{sec:problem}

Let $M$ be a pretrained language model (LM) with $L$ transformer layers and vocabulary~$\V$.  
For an input sequence $x_{1:T} = (x_1,\dots,x_T)$ and a decoding step $T+1$, we denote the layer–$\ell$ hidden representation by
$
h_{\ell}(x_{1:T}) \;\in\; \mathbb{R}^D,
$ where $ \ell \in \{1,\dots,L\},$
and $D$ is the model’s hidden size.
Unless stated otherwise, we drop the position subscript and write $h_\ell(x)$ for $h_{\ell}(x_{1:T})$.
The model’s next-token logit vector is
$
\phi(x_{1:T}) \;=\; W_{\!U}\,h_{L}(x_{1:T}) + b
 \in\mathbb{R}^{|\mathcal V|},
$
where $W_{\!U}\in\mathbb{R}^{|\mathcal V|\times D}$ is the unembedding and $b\in\mathbb{R}^{|\ V|}$ is the bias.  
Unless stated otherwise, we drop the position subscript and write $\phi(x)$ for $\phi(x_{1:T})$ by default.
The probability distribution for generating the next token is obtained via softmax,
\[
P_M(\,\cdot \mid x_{1:T}) \;=\; \mathrm{softmax}\!\bigl(\phi(x_{1:T})\bigr) 
\;\in\; \Delta^{|\mathcal V|-1},
\]
where $\Delta^{|\mathcal V|-1}$ denotes the probability simplex over~$\mathcal V$.
We assume a family of discrete latent linguistic properties $\Z=\{Z_1,Z_2,\dots\}$ defined on input texts. 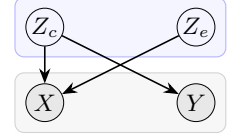
\begin{wrapfigure}{r}{0.4\linewidth}
\vspace{-.2cm}
    \centering
    \begin{tikzpicture}[
    >=Stealth,
    every node/.style={font=\small, align=center},
    latent/.style={circle, draw, line width=0.2pt, minimum size=5mm, inner sep=0pt},
    observed/.style={latent, fill=gray!18},
    edge/.style={->, semithick}
  ]
  \node[observed] (X)  at (0,0)     {$X$};
  \node[observed] (Y)  at (2,0)   {$Y$};
  \node[latent]   (Zc) at (0,1)   {$Z_c$};
  \node[latent]   (Ze) at (2,1) {$Z_e$};

  \draw[edge] (Zc) -- (X);
  \draw[edge] (Ze) -- (X);
  \draw[edge] (Zc) -- (Y);

  \begin{scope}[on background layer]
    \node[
      draw, rounded corners=1.2mm, inner sep=4pt,
      fill=gray!6, draw=gray!35,
      fit=(X) (Y),
    ] {};
    \node[
      draw, rounded corners=1.2mm, inner sep=4pt,
      fill=blue!4, draw=blue!30,
      fit=(Zc) (Ze),
    ] {};
  \end{scope}
\end{tikzpicture}
    \caption{$X$ is the sequence $x_{1:T}$  and $Y$ is the token at $T+1$. $Z_c$ and $Z_e$ are latent linguistic properties.
}
\vspace{-.5cm}
    \label{fig:properties_graph}
\end{wrapfigure} 
The classical ``probing'' paradigm evaluates whether a hidden vector $h_\ell(x)$
correlates with a linguistic property $Z \in \Z$ by training a supervised classifier on hidden representations $h_\ell(x)$.
Causal probing instead asks how or whether  manipulating the hidden representation at layer~$\ell$ changes a linguistic property $Z \in \Z$ and the model's next-token prediction denoted as ~$Y$.


Throughout the paper, we consider two types of properties. 
First, a \emph{causal} property $Z_c$ that directly governs the model’s next-token choice $Y$ at position $T+1$. 
For instance, in English subject–verb agreement, $Z_c\!\in\!\{\textsc{sg},\textsc{pl}\}$ encodes the subject number (singular or plural) and should impose the verb agreement at the next token prediction (e.g., $Y\in \{\textit{is}, \textit{are}\}$). 
Second, an \emph{environment} or \emph{nuisance} property $Z_e$ (e.g., the number of non-subject nouns in a prepositional phrase) that varies with the context and does not affect $Y$ (e.g., non-subject noun number does not affect the verb,
see Figure \ref{fig:properties_graph}).  
Both $Z_c$ and $Z_e$ take values in finite sets $\Kc$ and $\Ke$, respectively.

Let $f_\theta:\mathbb R^D\!\to\!\mathbb R^D$ be an interventional operator:  
\[
\tilde h_\ell(x)
\;=\;
f_\theta\!\bigl(h_\ell(x),\,z\!\rightarrow\!z'\bigr),
\quad\text{where }z'\!\neq\!z.
\]
This operator for a given $x$, intervenes on the hidden state to change the value of $Z$ from $z$ to $z'$




\section{Hidden‑state Driven Margin Intervention (HDMI)}
\label{sec:HDMI-loss}


Let $\phi:\R^{D}\!\to\!\R^{|\V|}$ be a differentiable logit
function (by default the LM head
$\phi(x)$; see Section~\ref{sec:problem}).
Consider the input sequence $x_{1:T}$
whose next–token distribution we want to predict. Let $h_\ell(x)$ denote the hidden state representation at layer~$\ell$. In the following examples, we shall denote the position $T+1$ by [MASK].

In causal probing, the dataset provides two mutually exclusive next–token continuations or labels:

\begin{itemize}
\item the \emph{source} token
      $v_a\in\V$, which is the token preferred by the model under its original distribution that realizes the factual value $z$
      of $Z_c$, and
\item the \emph{target}  token
      $v_b\in\V$, which realizes the counterfactual value $z'$
      of $Z_c$ we wish the model to adopt.
\end{itemize}
For instance, consider the sequence “\textit{The key [MASK] on the table}.” 
Without intervention, the base model might prefer $v_a=\texttt{is}$ at [MASK] because $z=\texttt{SG}$. 
If we wish to flip the subject‑number property to plural ($z'=\texttt{PL}$), the \emph{target} token at this step is $v_b=\texttt{are}$. 

Denote the indices of $v_a$ and $v_b$ in the logit vector $\phi$ by
$
\sigma \;=\; \mathrm{ID}(v_a),$ and $
\tau   \;=\; \mathrm{ID}(v_b).$
\paragraph{Single-token margin objective.}
We aim to find a perturbed hidden state
$\tilde h_\ell(x)$ that \emph{raises} the logit (or log-probability)
of the target token~$\tau$ while \emph{lowering} that of the
source token~$\sigma$ at the next decoding step.
We define the margin objective as follows:
\begin{equation}
\label{eq:HDMI-margin}
\mathcal L(x)
\;=\;
\phi(x)_{\tau} \;-\; \phi(x)_{\sigma}.
\end{equation}

Maximizing this margin objective increases the probability of the target token while decreasing that of the source token. We therefore apply gradient ascent on the hidden state $h_\ell(x)$. Note that for the final layer $L$, and for the $\phi$ defined in Section \ref{sec:problem}, because $\LL$ is linear in $h_L(x)$, computing the gradient is inexpensive:
\[
\nabla_{h_L} \mathcal L(x)
\;=\;
W_{\!U}^{\!\top}\bigl(e_{\tau}-e_{\sigma}\bigr),
\]

where $e_i$ is the $i$-th basis vector in $\R^{|\V|}$. For a general $\ell$:
$\nabla_{h_\ell} \mathcal L(x)
\;=\;
\bigl(\nabla_{h_\ell}\phi(x)\bigr)^{\!\top}
\bigl(e_{\tau}-e_{\sigma}\bigr),$
which requires computing a matrix–vector product, also a
computationally lightweight calculation.
{
To show the role of the margin objective in ~\eqref{eq:HDMI-margin} for reliably flipping the targeted property, we introduce the \emph{target-only} objective that promotes the logit of the target token but does not demote the source token:
\[
\mathcal{L}_{\text{target-only}}(x) \;=\; \phi(x)_{\tau}.
\]
The following theorem shows the theoretical justification of our margin loss and compare it with the target-only objective.
\begin{theorem}
\label{thm:main_margin_vs_target}

Let $z(h)=Wh+b$ be the final-layer logits and define the pairwise margin $m(h)=z_\tau(h)-z_\sigma(h)$. Moreover, let $d=w_\tau-w_\sigma$, where $w_i^\top$ is the $i$-th row of $W$. 
Under the constraint $\|\delta\|_2 \le \epsilon$, the maximum achievable margin increase satisfies $\max_{\|\delta\|_2 \le \epsilon} \big(m(h+\delta)-m(h)\big)=\epsilon\|d\|_2$, achieved by $\delta^\star=\epsilon\,d/\|d\|_2$. 
In contrast, the $\ell_2$-optimal target-only update $\delta_{\mathrm{TO}}=\epsilon\,w_\tau/\|w_\tau\|_2$ yields margin gain $m(h+\delta_{\mathrm{TO}})-m(h)=\epsilon\|d\|_2\cos\theta$, where $\theta$ is the angle between $d$ and $w_\tau$.

\end{theorem}

\begin{remark}
Note that $m(h)=z_\tau(h)-z_\sigma(h)$ coincides with our margin loss in~\eqref{eq:HDMI-margin}. 
Theorem~\ref{thm:main_margin_vs_target} shows that a target-only update may decrease this margin (if $\cos\theta<0$). 
Although such an update increases $z_\tau$, it does not control the behavior of $z_\sigma$. 
Consequently, the source logit may increase more than the target logit. 
This issue arises because the source term is omitted from the objective.

\end{remark}



}

\paragraph{Multi-token extension.}
Sometimes the property can be realized by several acceptable tokens at the very next step. For example, for subject-verb agreement, we may wish to use both \texttt{are} and \texttt{were} for the counterfactual value $z'=\textsc{PL}$, and  \texttt{is} and \texttt{was} for $z=\textsc{SG}$.

HDMI can optimize a set-based margin. 
Let
$
\mathcal T^{+}=\{\tau_1,\dots,\tau_p\},
\mathcal T^{-}=\{\sigma_1,\dots,\sigma_q\},
$
and define the loss as:
\begin{equation}
\label{eq:HDMI-set-margin}
\mathcal L_{\text{set}}(x)=
\sum_{i=1}^{p}\,\phi(x)_{\tau_i}-
\sum_{j=1}^{q}\,\phi(x)_{\sigma_j}.
\end{equation}

Let $u^{+}=\sum_{i} e_{\tau_i}$ and $u^{-}=\sum_{j} e_{\sigma_j}$ so that
$\mathcal{L}_{\text{set}}(x)=\phi(x)^{\!\top}(u^{+}-u^{-})$.
For interventions at the final layer ($\ell=L$),
$\nabla_{h_L}\,\mathcal{L}_{\text{set}}(x)
\;=\;
W_U^{\!\top}\bigl(u^{+}-u^{-}\bigr).$
The gradient for general $\ell$ is
$
\nabla_{h_\ell}\,\mathcal{L}_{\text{set}}(x)
\;=\;
\bigl(\nabla_{h_\ell}\phi(x)\bigr)^{\!\top}
\bigl(u^{+}-u^{-}\bigr),
$
which is the same as equation~\ref{eq:HDMI-margin} when
$p=q=1$. Starting from the original representation
$h^{(0)} = h_\ell(x)$, we perform $K$ steps of gradient ascent:
\begin{align}
g^{(k)}      \;=\; \nabla_{h^{(k)}} \LL(x), \qquad
h^{(k+1)}   \;=\; h^{(k)} + \alpha\, g^{(k)},
\end{align}
where $\alpha>0$ is the step size. The final state representation $\tilde h_\ell(x)=h^{(K)}$
is substituted back into the forward pass
at only layer~$\ell$ and position~$T+1$ for the next
token prediction; all other layers and decoding steps remain unchanged.

\section{Text editing with Lookahead HDMI}
\label{sec:HDMI6}
In this section, we show a use–case of HDMI in which for a given input sequence  
\(
x_{1:{T_\textit{in}}}=(x_1,\dots,x_{T_\textit{in}})
\), a \emph{user}
provides an edited sequence 
$\tilde{x}_{1:{T_\textit{ed}}}=(\tilde x_1,\dots,\tilde x_{T_\textit{ed}})$. Let $T=\min({T_\textit{in}},{T_\textit{ed}})$. The goal is to change the input sequence toward the edited version fluently. 
Rather than forcing the LM to reproduce $\tilde{x}_i$s, we
\emph{steer} the hidden states of the input sequence so that the
generated sequence remains fluent while
implementing the user’s changes.
For every decoding step $t+1\in\{1,\dots,T\}$, we denote the last-token layer-$L$ hidden $h_{L}(x_{1:t})$ by $h_t\in\mathbb R^D$, the temperature by $\beta_g$, the next-token logit $\phi(x_{1:t})$ by $\phi_t$ and the next-token distribution by $y_t$. Then,
\[
\phi_t \;=\; W_{\!U}h_t + b,
\qquad
y_t \;=\; \mathrm{softmax}\!\bigl(\phi_t / \beta_g\bigr).
\]

Since we considered an autoregressive LM, a hidden state $h_t$, is not aware of future edits. To let the $h_t$ be influenced by \emph{future} edit positions ($t'>t$), we must consider the transition from $h_t$ to $h_{t+1}$ through the expected embedding $m_t$ as follows:
\[
h_t \;\rightarrow\; 
\phi_t=\; W_{\!U}h_t + b \;\xrightarrow\; 
y_t \;=\; softmax\!\bigl(\phi_t/\beta_g\bigr)
\]
\[\;\rightarrow\;
m_t \;=\; E^{\!\top}y_t
\;\rightarrow\;
h_{t+1}=\Trsoft(m_t)
\]
where $E\in\mathbb{R}^{|\mathcal V|\times d_{e}}$ is the token embedding matrix, $d_{e}$ is the embedding size, and $m_t\in\mathbb{R}^{d_e}$ is the expected value of the next-token prediction based on the distribution $y_t$. 
\(\Trsoft\) denotes the decoder one-step
transition that maps the expected next-token embedding to the next last-layer hidden state. Note that \(\Trsoft\) maintains a cache of all past tokens' attention states, which we have omitted here for brevity.  In this formulation, gradients of a margin objective (such as equation \ref{eq:HDMI-margin}) can be backpropagated via a vector–Jacobian product (VJP) through
the softmax–expected‑embedding–transition chain to the hidden states of the past decoding steps ($t'<t)$.  Since we do a forward pass with $m_t$, no gradient flows
through argmax/sampling, and therefore, the transitions remain differentiable.
Note that autoregressive LLMs were never trained to consume combinations of expected embeddings; hence, feeding expected embeddings in their token generation process (forward pass) puts the model off‑manifold and quickly collapses fluency--even without intervening on hidden states.
To preserve fluency while still obtaining useful gradients, we decouple the forward token generation process from the gradient backpropagation path. At each decoding step, we choose a very low‑temperature $\beta_f$ in softmax and feed that generated embedding into the model, to mimic exactly the standard inference. In parallel, to compute the gradient with backpropagation, starting from the current hidden state, we build a differentiable ``lookahead'' objective by using the expected embedding under a high‑temperature value ($\beta_g$ near 1). In the following, we define the objective function for text editing.

Denote the set of positions that must change (specified by user edits) by
$
\mathcal M \;=\;
\bigl\{\,j\in\{1,\dots,T\}\!:\;x_j\neq\tilde x_j\bigr\},
$
and let $(a_j,b_j)=\bigl(\mathrm{ID}(x_j),\mathrm{ID}(\tilde x_j)\bigr)$ be their respective indices.  
Before predicting $x_{t+1}$, we build an objective that considers the edits up to $S_{\max}$ steps ahead:
\begin{align}
J_t(h_t)
\;=\;
\sum_{s=1}^{S_{\max}} &\,
\mathbf 1_{\{t+s\in\mathcal M\}}
\bigl[\phi(x_{1:t+s-1})_{b_{t+s}} \\
&- \phi(x_{1:t+s-1})_{a_{t+s}}\bigr].
\end{align}
Maximizing this cumulative margin objective adjusts $h_t$ so that the probability of future target tokens increases while decreasing that of the source ones.
To prevent deviating from the input sequence in the not edited parts and follow the input sequence logit, we also incorporate the source token logit at position $t+1$ and regulate it using a coefficient $\lambda_{\text{fact}}$:
\[
J_t \;\leftarrow\;
J_t + \lambda_{\text{fact}}\;\phi\bigl(x_{1:t}\bigr)_{a_{t+1}},
\qquad
\lambda_{\text{fact}}\in[0,1].
\]

Starting from $h_t^{(0)}=h_t$, we perform $K$ steps of gradient ascent
\[
g_t^{(k)}=\nabla_{h_t^{(k)}}J_t,
\quad
h_t^{(k+1)}=h_t^{(k)}+\alpha\, g_t^{(k)},
\]
where $\alpha$ is the step size.  
After $K$ inner steps we obtain $h'_t=h_t^{(K)}$, and form the following distribution using a low temperature $\beta_f$
\[
y'_t=\mathrm{softmax}\!\bigl(W_{\!U}h'_t/\beta_f\bigr).
\]
We show the next token by taking the argmax of $y'_t$, $x_{t+1}^\star=\arg\max y'_t$ (just to show); ~however, we continue the forward pass with the \emph{expected embedding} $m'_t = E^\top y'_t$.

\section{Experiments}
\label{sec:exp-setup}

\subsection{Causal Probing with HDMI}
\label{sec:exp-cprob}
We study pretrained decoder–only LMs with $L$ transformer layers. The complete code is provided in the
supplementary material.
Consider the input sequence $x_{1:T}$
whose next–token distribution we want to predict. Unless noted otherwise, we intervene at the final layer $\ell\!=\!L$. 
In this section, we often denote $h_{\ell}(x_{1:T})$ by $h_\ell$ for sake of brevity.


\paragraph{Datasets.}
We evaluate our approach on two complementary sources.

\textbf{LGD agreement corpus.}  
  We follow the protocol of \citet{canby2024reliable}: natural Wikipedia sentences from the original LGD dataset \citep{linzen2016assessing} are filtered so that both singular and plural inflections of the target verb are in the dataset.  
  Each dataset sample supplies (i) an input sequence ending right before the next-token prediction (the model predicts the verb denoted with [MASK] in the prompt), (ii) mutually exclusive labels $\langle v_{\text{sg}},v_{\text{pl}}\rangle$, that we choose $v_a$ and $v_b$ from (e.g., \textit{locks} and \textit{lock} respectively), and (iii) annotations for the subject number $Z_c$, which is either singular (SG) or plural (PL) and, when present, the number of the most recent non-subject noun in a prepositional phrase ($Z_e$). We map $Z_c\!\in\!\{\textsc{sg},\textsc{pl}\}$ to $\{0,1\}$ and {$Z_e$} to a 3–class label $\{0\!=\!\varnothing,1\!=\!\textsc{sg},2\!=\!\textsc{pl}\}$.

\textbf{CausalGym.}
CausalGym \citep{arora2402causalgym} groups examples into \emph{suites}, each targeting a single grammatical phenomenon (e.g., agreement with prepositional phrase (PP) distractors, subordination, clefting, filler–gap).  
Items are provided as \emph{minimal pairs} $\langle x_{\text{src}} , x_{\text{cf}}\rangle$, and their corresponding labels $\langle y_{\text{src}}, y_{\text{cf}}\rangle$ that are identical except for the property of interest $Z_c$, ensuring that any difference in model preference can be attributed to this property. For instance, $\langle x_{\text{src}} , x_{\text{cf}}\rangle$=$\langle \textit{John walked because [MASK]} , \textit{Jane walked } \\\textit{because [MASK]}\rangle $ and $\langle y_{\text{src}}, y_{\text{cf}}\rangle$= $\langle \textit{he, she}\rangle $.
From each pair, we create two samples by swapping which label is active; hence, the source/target tokens (defined in Section \ref{sec:HDMI-loss}) are well defined. For example, for the prompt \textit{John walked because [MASK]}, the source token $v_a$ is \textit{he} and the target one $v_b$ is \textit{she}. 
We take the varied feature as $Z_c$ (in this example ``gender'') and similar to LGD dataset, we map $Z_c\!\in\!\{\textit{male},\textit{female}\}$ to $\{0,1\}$. For the environment property $Z_e$ we use a preposition–family heuristic over the prompt near the verb: 
$Z_e \in \{0,1,2,3,4\} \equiv \{\textsc{none},\textsc{of},\textsc{in},\textsc{with/by},\textsc{other}\},
$
obtained by scanning the last 12 tokens for the most recent preposition.

\begin{table*}[htp!]
\centering
\small
\caption{CausalGym and LGD: Completeness, Selectivity, and Reliability for HDMI (ours), AlterRep, FGSM, and PGD on Meta-Llama-3-8B-Instruct and EleutherAI/pythia-70m. Higher value is better. C. , S., and R. are short for Completeness, Selectivity, and Reliability, respectively.}
\label{tab:combined_results_4methods}
\begin{tabular}{llcccccc}
\toprule
Task & Method & \multicolumn{3}{c}{Meta-Llama-3-8B-Instruct} & \multicolumn{3}{c}{EleutherAI/pythia-70m} \\
 &  & C. & S. & R. & C. & S. & R. \\
\midrule

\shortstack[l]{\texttt{agr\_sv\_num\_}\\\texttt{obj-relc}} & HDMI     & 1.0000 & 1.0000 & 1.0000 & 1.0000 & 1.0000 & 1.0000 \\
                                & AlterRep & 0.9800 & 1.0000 & 0.9899 & 0.9600 & 1.0000 & 0.9796 \\
                                & FGSM     & 1.0000 & 1.0000 & 1.0000 & 1.0000 & 1.0000 & 1.0000 \\
                                & PGD      & 0.6307 & 1.0000 & 0.7736 & 1.0000 & 1.0000 & 1.0000 \\
\midrule
\shortstack[l]{\texttt{agr\_sv\_} \\\texttt{num\_pp}}      & HDMI     & 1.0000 & 0.9362 & 0.9671 & 0.9984 & 0.9409 & 0.9688 \\
                                & AlterRep & 0.9500 & 0.8680 & 0.9072 & 0.5000 & 0.5108 & 0.5053 \\
                                & FGSM     & 1.0000 & 0.8314 & 0.9080 & 0.9984 & 0.4906 & 0.6579 \\
                                & PGD      & 0.3650 & 0.3648 & 0.3649 & 1.0000 & 0.8550 & 0.9218 \\
\midrule
\shortstack[l]{\texttt{agr\_refl\_num}\\\texttt{\_subj-relc}} & HDMI  & 1.0000 & 1.0000 & 1.0000 & 1.0000 & 1.0000 & 1.0000 \\
                                   & AlterRep & 0.9898 & 1.0000 & 0.9949 & 0.5000 & 1.0000 & 0.6667 \\
                                   & FGSM     & 1.0000 & 1.0000 & 1.0000 & 0.4963 & 1.0000 & 0.6634 \\
                                   & PGD      & 0.9636 & 1.0000 & 0.9815 & 0.4912 & 1.0000 & 0.6588 \\
\midrule
\shortstack[l]{\texttt{agr\_refl\_}\\\texttt{num\_pp}}    & HDMI     & 0.9902 & 0.8272 & 0.9014 & 1.0000 & 0.8983 & 0.9464 \\
                                & AlterRep & 0.9900 & 0.6540 & 0.7877 & 0.5449 & 0.8843 & 0.6743 \\
                                & FGSM     & 0.5341 & 0.9881 & 0.6934 & 0.9642 & 0.9206 & 0.9419 \\
                                & PGD      & 0.3335 & 0.5155 & 0.4050 & 0.6142 & 1.0000 & 0.7610 \\
\midrule
\shortstack[l]{\texttt{gss\_subord\_}\\\texttt{subj-relc}} & HDMI     & 1.0000 & 0.8969 & 0.9456 & 1.0000 & 1.0000 & 1.0000 \\
                                & AlterRep & 0.9700 & 0.8938 & 0.9303 & 0.6095 & 1.0000 & 0.7573 \\
                                & FGSM     & 0.5500 & 1.0000 & 0.7097 & 0.1507 & 1.0000 & 0.2619 \\
                                & PGD      & 0.6411 & 1.0000 & 0.7813 & 0.3463 & 1.0000 & 0.5145 \\
\midrule
\shortstack[l]{\texttt{gss\_subord\_}\\\texttt{pp}}        & HDMI     & 1.0000 & 0.8969 & 0.9456 & 0.9984 & 0.7371 & 0.8481 \\
                                & AlterRep & 0.9700 & 0.8938 & 0.9303 & 0.6896 & 0.7051 & 0.6972 \\
                                & FGSM     & 1.0000 & 0.9382 & 0.9681 & 0.1333 & 0.9799 & 0.2347 \\
                                & PGD      & 0.9950 & 0.5298 & 0.6914 & 0.3015 & 0.6628 & 0.4144 \\
\midrule
\texttt{cleft}                  & HDMI     & 1.0000 & 0.9800 & 0.9899 & 1.0000 & 1.0000 & 1.0000 \\
                                & AlterRep & 0.8603 & 0.9800 & 0.9163 & 0.7300 & 1.0000 & 0.8440 \\
                                & FGSM     & 1.0000 & 1.0000 & 1.0000 & 0.7900 & 1.0000 & 0.8827 \\
                                & PGD      & 0.9946 & 1.0000 & 0.9973 & 0.9599 & 1.0000 & 0.9796 \\
\midrule
\shortstack[l]{\texttt{filler\_gap\_}\\\texttt{hierarchy}} & HDMI     & 1.0000 & 0.8414 & 0.9138 & 1.0000 & 1.0000 & 1.0000 \\
                                & AlterRep & 0.9900 & 0.8443 & 0.9114 & 0.8299 & 1.0000 & 0.9070 \\
                                & FGSM     & 1.0000 & 1.0000 & 1.0000 & 1.0000 & 1.0000 & 1.0000 \\
                                & PGD      & 0.8300 & 1.0000 & 0.9071 & 1.0000 & 1.0000 & 1.0000 \\
\midrule
\texttt{filler\_gap\_pp}        & HDMI     & 0.9710 & 0.4412 & 0.6067 & 0.7404 & 0.6275 & 0.6793 \\
                                & AlterRep & 0.8095 & 0.7389 & 0.7726 & 0.5100 & 0.4582 & 0.4782 \\
                                & FGSM     & 0.9764 & 0.0501 & 0.0954 & 0.4453 & 0.9951 & 0.6153 \\
                                & PGD      & 0.8480 & 0.2290 & 0.3606 & 0.4561 & 0.9310 & 0.6122 \\
\midrule
\shortstack[l]{\texttt{filler\_gap\_}\\\texttt{subj}}      & HDMI     & 1.0000 & 0.4651 & 0.6349 & 1.0000 & 1.0000 & 1.0000 \\
                                & AlterRep & 1.0000 & 0.5200 & 0.6842 & 1.0000 & 0.5200 & 0.6842 \\
                                & FGSM     & 0.9200 & 1.0000 & 0.9583 & 0.3435 & 1.0000 & 0.5114 \\
                                & PGD      & 0.9050 & 1.0000 & 0.9501 & 0.1787 & 1.0000 & 0.3033 \\
\midrule
\texttt{LGD}                    & HDMI     & 0.9412 & 0.8117 & 0.8716 & 0.9341 & 0.8538 & 0.8921 \\
                                & AlterRep & 0.9490 & 0.6536 & 0.7741 & 0.9951 & 0.3234 & 0.4881 \\
                                & FGSM     & 0.5813 & 0.3337 & 0.4240 & 0.4393 & 0.9959 & 0.6097 \\
                                & PGD      & 0.5402 & 0.4149 & 0.4694 & 0.7124 &  0.532 & 0.6091 \\
\bottomrule
\end{tabular}
\end{table*}

A validation probe \citep{canby2024reliable} is a lightweight classifier trained to \emph{decode} a latent property from hidden states, without participating in the intervention itself. 
Given layer–$\ell$ representations $h_\ell(x)$, we train two probes on a disjoint split:
(i) $\vZc$, which estimates $P(Z_c\!\mid\!h_\ell)$ and is used to quantify \textbf{completeness} by checking whether an intervention shifts the encoding toward the counterfactual value; and 
(ii) $\vZe$, which estimates $P(Z_e\!\mid\!h_\ell)$ and is used to quantify \textbf{selectivity} by checking that non‑target properties remain stable. 
Importantly, validation probes are never trained on intervened representations and never see the test set; their role is to read out what the model encodes, not to learn to cope with interventions.

The data is partitioned into (1) an \emph{interventional} split, (2) a \emph{validation‑probe} split, and (3) a held‑out \emph{test} split for two reasons. 
First, to avoid leakage: the interventional split is used to fit any probe-driven intervention (e.g., a $Z_c$ probe for gradient‑based baselines) and to tune intervention hyper‑parameters, while validation probes are fit \emph{only} on the validation‑probe split. 
Second, to obtain unbiased metrics: we subsample the validation‑probe split so that $Z_c$ and $Z_e$ are approximately independent \citep{canby2024reliable}, ensuring $\vZc$ cannot spuriously use $Z_e$ (and vice versa). 
The test split is used exclusively for final reporting of completeness, selectivity, and reliability after all training and tuning are performed.

We compared our method (HDMI) with several baselines:
 \textbf{HDMI (ours).} Gradient ascent on the next–token margin (~\eqref{eq:HDMI-margin}), i.e.\ increase $\phi(x)_\tau$ while  decreasing $\phi(x)_\sigma$. In Appendix \ref{sec:ablation-target-only}, there is an ablation study on only increasing $\phi(x)_\tau$.
 \textbf{GBI.} Gradient–based counterfactual intervention via FGSM/PGD \citep{goodfellow2015explaining,madry2018towards} against the interventional $Z_c$ probe, targeted to predict a counterfactual value of $Z_c$  within an $\ell_\infty$ or $\ell_2$ ball (default: $\ell_\infty$).
 \textbf{AlterRep \citet{ravfogel2021counterfactual}.} It modifies only the component of the representation $h_\ell$ that lies in the row space of a linear concept classifier (the span of its weight vectors) and leaves the part orthogonal to that space unchanged.



\begin{figure}[htp]
  \centering
    \centering
    \includegraphics[width=0.99\linewidth]{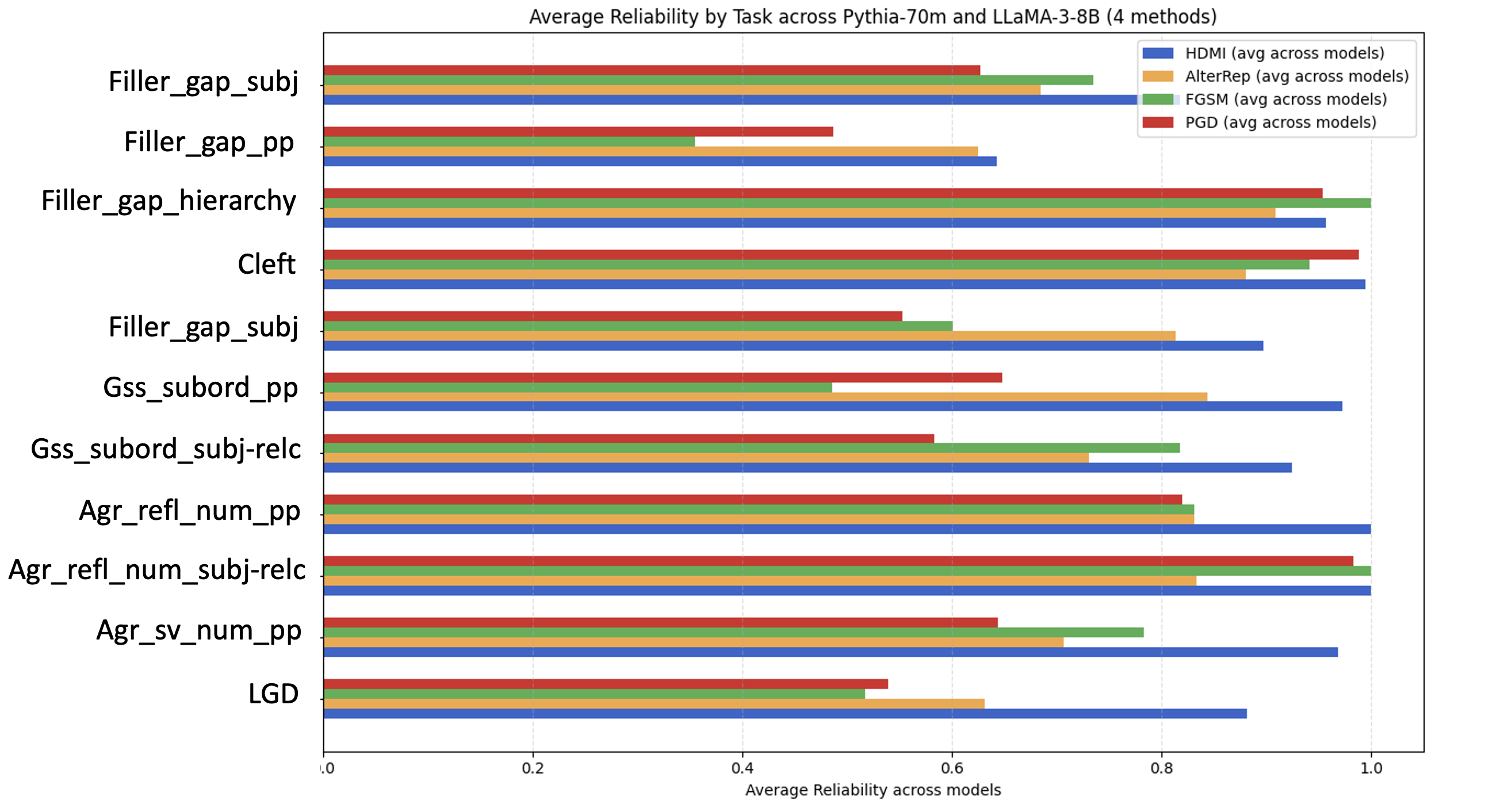}
    \caption{Reliability by tasks. Reliability is averaged across LLaMA and Pythia-70M models. }
    \label{fig:total_reliability}
  \label{fig:pythia70m_overview}
\end{figure}
Following \citet{canby2024reliable}, we report:

 \textbf{Completeness.} Let $p_c^{\text{after}}=\mathrm{softmax}(\vZc(\tilde h_\ell(x)))$, where the softmax is taken over $|\Kc|$ classes.
The desired distribution after applying the interventional operator $f_\theta\!\bigl(h_\ell(x),\,z\!\rightarrow\!z'\bigr)$ is the one–hot $e_{z'} \in \{0,1\}^{|\Kc|}$, where $ \bigl[e_{z'}\bigr]_k=1$ if $k=z'$, and zero otherwise; we define
\[
\mathrm{Comp} \;=\; 1 - d_{\mathrm{TV}}\!\bigl(p_c^{\text{after}},\, e_{z'}\bigr),
\qquad
d_{\mathrm{TV}}(p,q)=\tfrac12\|p-q\|_1,
\]

where $d_{\mathrm{TV}}(p,q)$ is the total variation (TV) distance between distributions $p$ and $q$.

 \textbf{Selectivity.} We measure selectivity with TV distance between the post- and pre-intervention
validation-probe distributions and normalize by the maximum possible
TV shift, which is defined as $m(p)$ below. Let $p_e^{\text{before}}=\mathrm{softmax}(\vZe(h_\ell(x)))$ and $p_e^{\text{after}}=\mathrm{softmax}(\vZe(\tilde h_\ell(x)))$.
We report the selectivity score
\[
\mathrm{Sel}
\;=\;
1 - \frac{d_{\mathrm{TV}}(p_e^{\text{after}},\,p_e^{\text{before}})}{m(p_e^{\text{before}})},\]
\[
m(p)=\max\!\bigl\{1-\min_i p_i,\;\max_i p_i\bigr\}.
\]


\textbf{Reliability.} The harmonic mean
$\mathrm{Rel}
\;=\;
\frac{2\,\mathrm{Comp}\times \mathrm{Sel}}{\mathrm{Comp}+\mathrm{Sel}}.$



For each split, we hold out a random $20\%$ subset of that split as internal validation to report the probe accuracy. The probe accuracy for our experiments in different tasks is above 90\% except for some tasks on
the Pythia-70M model, where the accuracy of the interventional $Z_c$ probe (evaluated on its holdout within the interventional split) dropped to approximately 70\%, which made AlterRep/FGSM/PGD updates unstable. To stabilize these baselines, we adjusted the train/validation ratio within the interventional split, allocating a larger share to probe training and a smaller share to its internal validation, which improved probe accuracy. It is noteworthy to emphasize that  HDMI is probe-free and does not rely on an interventional probe at all; consequently, its performance is independent of the size of the interventional split and of interventional probe accuracy.

Table \ref{tab:combined_results_4methods} (on LLaMA‑3‑8B‑Instruct and Pythia‑70M) reports Completeness, Selectivity, and Reliability across LGD and the CausalGym suites for four methods: HDMI (ours), AlterRep, FGSM, and PGD. More experiments are provided in Appendix \ref{sec:extra_exp}. Figure~\ref{fig:total_reliability} summarizes the average Reliability across models for all four methods.


In agreement suites (\texttt{agr\_*}), HDMI is consistently strong. It achieves perfect Completeness and higher Reliability except in \texttt{agr\_refl\_num\_pp} suite, where FGSM outperforms in Selectivity. On subordination and clefting, HDMI remains competitive or better in Reliability. 
On LLaMA's subordination suites, FGSM is often strong and PGD's performance is mixed; both drop substantially on Pythia‑70M.
HDMI remains consistent across models.
Filler–gap remains the most challenging family. On \texttt{filler\_gap\_pp} (LLaMA), AlterRep surpasses HDMI, while FGSM/PGD underperform markedly. In contrast, on Pythia‑70M \texttt{filler\_gap\_pp}, HDMI is stronger than AlterRep, and FGSM/PGD performs slightly better. Reliability of HDMI is consistent across models, while others show mixed Reliability and are sensitive to the quality of the \emph{interventional} $Z_c$ probe they target. 

Taken together, the results indicate that a simple, per instance margin ascent (\eqref{eq:HDMI-margin}) provides a strong performance-high Completeness with competitive Selectivity-translating into higher Reliability on most tasks and across models (Figure~\ref{fig:total_reliability}). The remaining gaps, notably in \texttt{filler\_gap\_pp} on LLaMA, suggest promising directions for the multi‑token objective (\eqref{eq:HDMI-set-margin}), finer layer selection, and step‑size scheduling to further improve Selectivity without sacrificing Completeness.

\subsection{Lookahead HDMI}
In this section, we show text-editing examples with LA-HDMI.
Figure \ref{fig:ctf}(a) illustrates LA-HDMI editing with two simultaneous token substitutions (\emph{girl}$\!\rightarrow\!$\textbf{\emph{owl}}, \emph{stars}$\!\rightarrow\!$\textbf{\emph{sun}}). HDMI applies a next-step, head-aware margin ascent at each decoding step with the expected embedding, allowing gradients to ``look ahead'' through the softmax–embedding–transition path (Sec.~\ref{sec:HDMI6}). In this example, lookahead successfully steers the model to realize the user’s edits while retaining fluency, including the necessary local contextual adjustments (e.g., the preceding article for \textit{owl}) without over-editing the rest of the sentence. Figure \ref{fig:ctf}(b) shows another example where the user asked for change \emph{was}$\!\rightarrow\!$\textbf{\emph{were}}. LA-HDMI performs lookahead steering before generating \emph{he} and puts \emph{we} instead, which is grammatically correct.

In general, we observe that this lookahead mechanism is sensitive to hyperparameters such as the horizon $S_{\max}$, temperatures $\beta_g$ and $\beta_f$, step size $\alpha$, and regularization coefficient $\lambda_{\text{fact}}$. When the token that must be adapted is not close to the edited token---for example a determiner, auxiliary, or other agreement carrier several positions away---the product of Jacobians along the expected-embedding path can attenuate the signal, and the gradient may vanish before reaching the earlier position. In such cases, careful fine-tuning (e.g., modestly increasing $S_{\max}$, using a slightly higher $\beta_g$ to reduce peaky distributions, or scheduling $\alpha$) tends to improve edit realization without sacrificing fluency.
There are some cases where the lookahead mechanism fails. This is an interesting direction for future work. Please refer to Appendix \ref{sec:extra_exp} for more details. {
Note that the numeric evaluation of HDMI applies for LA‑HDMI as well since LA-HDMI is locally applying HDMI over a sequence of tokens. Showing ``quantitative'' evaluation over the full generated sequence for LA-HDMI is inherently dependent on the input sequence because the user deliberately chooses how far they want the generated sequence be deviated from the proposed edits (to preserve coherency, being natural or fictional, etc) by adjusting the hyperparameters $\lambda_{\text{fact}}$ and $\alpha$. Therefore, there is a tradeoff between how coherent the text is and how strictly the user wants to apply the edits. Moreover,  currently, no standardized dataset for user-specified multi‑token edits with ground‑truth labels exists; building such a benchmark is valuable but out of scope of this work. 
    }

\begin{figure}[t]
    \centering
        \begin{subfigure}[t]{0.48\textwidth}
        \centering
        \includegraphics[width=\linewidth]{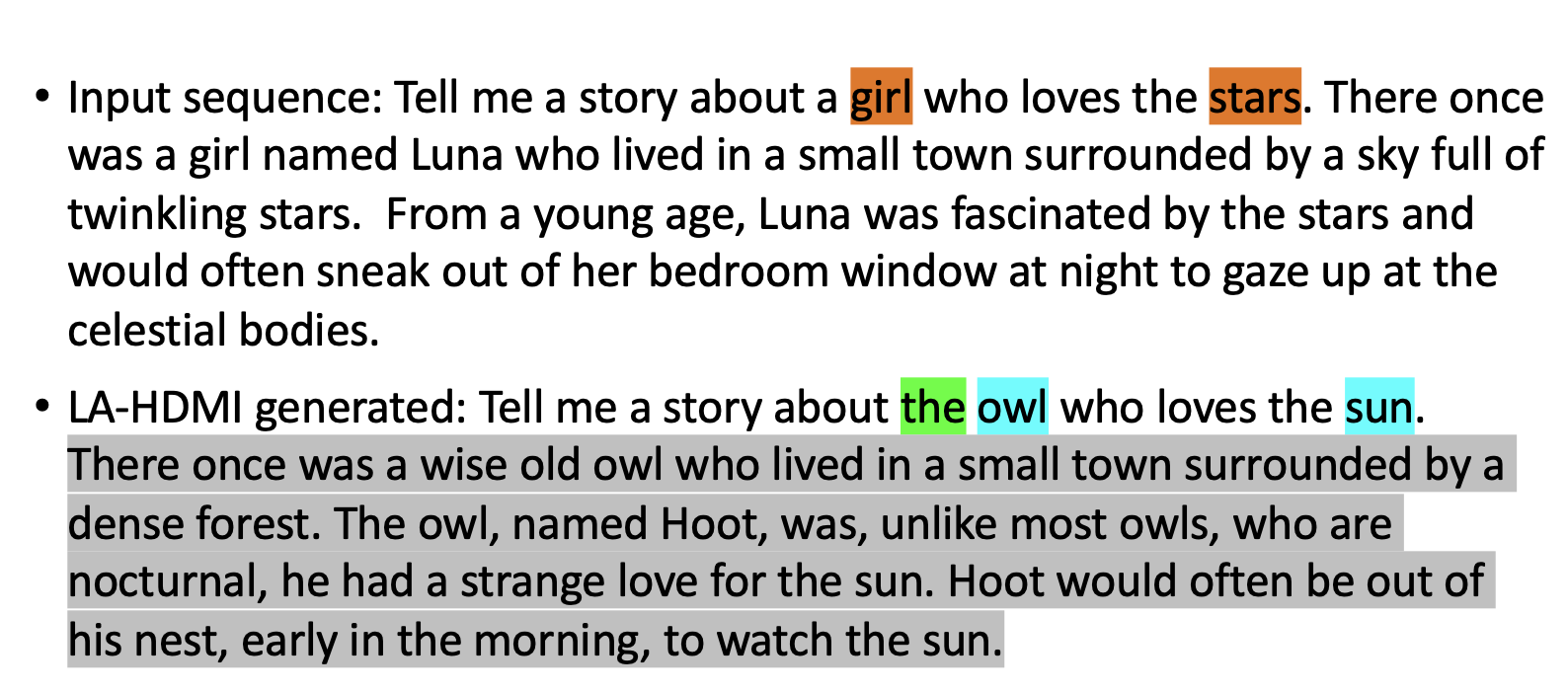}
        \caption{\textbf{}
        The input text
        (top) contains the source words \emph{girl} and \emph{stars}.
        The user supplies inline edits, striking out the source
        tokens and inserting \textbf{\emph{owl}} and \textbf{\emph{sun}}.
        LA-HDMI steers the hidden state of not only these words but also the words before to produce a fluent bottom sentence. Hence
        the model generates \textbf{\emph{the}} instead of \textbf{\emph{a}} while keeping the rest of the wording consistent.}
        \label{fig:ctf-b}
    \end{subfigure}
    \hfill
    \begin{subfigure}[t]{0.48\textwidth}
        \centering
        \includegraphics[width=\linewidth]{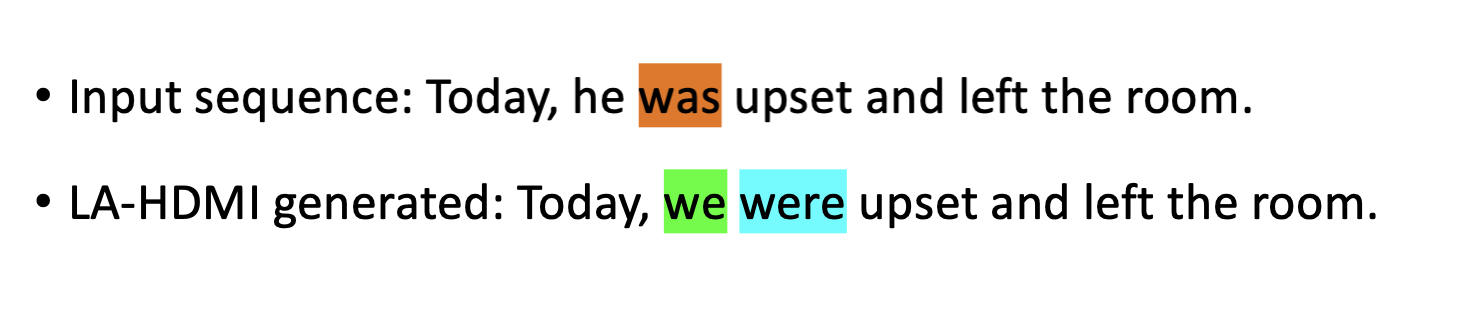}
        \caption{\textbf{}
        The input text (top) contains the source word
        \emph{was}. The user supplies inline edits with
        target word \textbf{\emph{were}}.
        LA-HDMI performs lookahead steering of the hidden state, hence
        the model generates \textbf{\emph{we}} beforehand.}
        \label{fig:ctf-a}
    \end{subfigure}

    \caption{HDMI editing examples. Left (a) and right (b) show two
    different edit realizations.}
    \label{fig:ctf}
    \end{figure}

\section{Conclusion}
We introduced HDMI, a probe-free, inference-time causal probing method that modifies hidden states via a simple logit-margin, and LA-HDMI for fluent, lookahead text editing. HDMI attains high completeness and selectivity, yielding strong reliability versus probe-driven baselines. Limitations include the sensitivity of hyperparameters and attenuation of lookahead gradients in text editing. Future work includes adaptive layer selection, step-size scheduling, and broader applications of the multi-token objective.






\bibliography{main}

\newpage

\onecolumn

\title{Supplementary Material}
\maketitle
\appendix
\section{More experiments}
\label{sec:extra_exp}

\subsection{more causalgym tasks}
More experiments on causalgym tasks are provided below:
On \texttt{agr\_gender}, all methods are near ceiling on both models, with HDMI matching the best. On \texttt{cleft} and \texttt{filler\_gap\_hierarchy}, methods are generally at or near ceiling; HDMI ties or leads in most cases.
\begin{table*}[htp!]
\centering
\small
\caption{CausalGym and LGD: Completeness, Selectivity, and Reliability for HDMI (ours), AlterRep, FGSM, and PGD on Meta-Llama-3-8B-Instruct and EleutherAI/pythia-70m. Higher is better.}
\label{tab:extra_combined_results_4methods}
\begin{tabular}{llcccccc}
\toprule
Task & Method & \multicolumn{3}{c}{Meta-Llama-3-8B-Instruct} & \multicolumn{3}{c}{EleutherAI/pythia-70m} \\
 &  & Completeness & Selectivity & Reliability & Completeness & Selectivity & Reliability \\
\midrule

\shortstack[l]{\texttt{agr\_sv\_num\_}\\\texttt{obj-relc}} & HDMI     & 1.0000 & 1.0000 & 1.0000 & 1.0000 & 1.0000 & 1.0000 \\
                                & AlterRep & 0.9800 & 1.0000 & 0.9899 & 0.9600 & 1.0000 & 0.9796 \\
                                & FGSM     & 1.0000 & 1.0000 & 1.0000 & 1.0000 & 1.0000 & 1.0000 \\
                                & PGD      & 0.6307 & 1.0000 & 0.7736 & 1.0000 & 1.0000 & 1.0000 \\
\midrule
\texttt{agr\_gender}            & HDMI     & 1.0000 & 1.0000 & 1.0000 & 1.0000 & 1.0000 & 1.0000 \\
                                & AlterRep & 1.0000 & 1.0000 & 1.0000 & 0.5000 & 1.0000 & 0.6667 \\
                                & FGSM     & 1.0000 & 1.0000 & 1.0000 & 1.0000 & 1.0000 & 1.0000 \\
                                & PGD      & 0.9350 & 1.0000 & 0.9664 & 1.0000 & 1.0000 & 1.0000 \\
\midrule

\shortstack[l]{\texttt{npi\_any}\\\texttt{\_subj-relc}}    & HDMI     & 1.0000 & 1.0000 & 1.0000 & 1.0000 & 1.0000 & 1.0000 \\
                                & AlterRep & 1.0000 & 1.0000 & 1.0000 & 0.5000 & 1.0000 & 0.6667 \\
                                & FGSM     & 1.0000 & 1.0000 & 1.0000 & 0.5000 & 1.0000 & 0.6667 \\
                                & PGD      & 0.9600 & 1.0000 & 0.9796 & 0.4999 & 1.0000 & 0.6666 \\
\midrule

\texttt{cleft}                  & HDMI     & 1.0000 & 0.9800 & 0.9899 & 1.0000 & 1.0000 & 1.0000 \\
                                & AlterRep & 0.8603 & 0.9800 & 0.9163 & 0.7300 & 1.0000 & 0.8440 \\
                                & FGSM     & 1.0000 & 1.0000 & 1.0000 & 0.7900 & 1.0000 & 0.8827 \\
                                & PGD      & 0.9946 & 1.0000 & 0.9973 & 0.9599 & 1.0000 & 0.9796 \\
\midrule
\shortstack[l]{\texttt{filler\_gap\_}\\\texttt{hierarchy}} & HDMI     & 1.0000 & 0.8414 & 0.9138 & 1.0000 & 1.0000 & 1.0000 \\
                                & AlterRep & 0.9900 & 0.8443 & 0.9114 & 0.8299 & 1.0000 & 0.9070 \\
                                & FGSM     & 1.0000 & 1.0000 & 1.0000 & 1.0000 & 1.0000 & 1.0000 \\
                                & PGD      & 0.8300 & 1.0000 & 0.9071 & 1.0000 & 1.0000 & 1.0000 \\
\midrule

\texttt{LGD}                    & HDMI     & 0.9412 & 0.8117 & 0.8716 & 0.9341 & 0.8538 & 0.8921 \\
                                & AlterRep & 0.9490 & 0.6536 & 0.7741 & 0.9951 & 0.3234 & 0.4881 \\
                                & FGSM     & 0.5813 & 0.3337 & 0.4240 & 0.4393 & 0.9959 & 0.6097 \\
                                & PGD      & 0.5402 & 0.4149 & 0.4694 & 0.5402 & 0.4149 & 0.4694 \\
\bottomrule
\end{tabular}
\end{table*}

\subsection{Ablation: Removing the Margin Term (Target-only Objective)}
\label{sec:ablation-target-only}

We hypothesize that the margin objective in ~\eqref{eq:HDMI-margin} is critical for reliably flipping the targeted property. To test this, we removed the source target term and optimized a \emph{target-only} objective that promotes the logit of the target token but does not explicitly demote the source token:
\[
\mathcal{L}_{\text{target-only}}(x) \;=\; \phi(x)_{\tau}
\quad\text{vs.}\quad
\mathcal{L}_{\text{margin}}(x) \;=\; \phi(x)_{\tau} - \phi(x)_{\sigma}.
\]
We evaluated on the LGD corpus with Meta\mbox{-}Llama\mbox{-}3\mbox{-}8B\mbox{-}Instruct. The results are summarized below and compared against the original HDMI numbers reported in Table~\ref{tab:combined_results_4methods}.

\begin{table}[t!]
\centering
\small
\caption{LGD (Meta\mbox{-}Llama\mbox{-}3\mbox{-}8B\mbox{-}Instruct) ablation of the HDMI objective. The margin-based HDMI row reproduces the LGD numbers from Table~\ref{tab:combined_results_4methods}.}
\begin{tabular}{lccc}
\toprule
Objective & Completeness & Selectivity & Reliability \\
\midrule
Target-only ($\phi_\tau$) & 0.7407 & 0.8145 & 0.7758 \\
Margin (\eqref{eq:HDMI-margin}) & 0.9412 & 0.8117 & 0.8716 \\
\midrule
$\Delta$ (target-only $-$ margin) & $-0.2005$ & $+0.0028$ & $-0.0958$ \\
\bottomrule
\end{tabular}

\label{tab:ablation_margin}
\end{table}

\subsection{Case study when lookahead backpropagation fails}
Consider the following input, edited, and generated sequences.
\begin{itemize}
    \item \emph{Factual}: \textit{Today, the boy with the red shirt took his own bag and left the room.}\\
    \item \emph{Edited}: \textit{Today, the boy with the red shirt took her own bag and left the room.}\\
    \item \emph{Generated}: \textit{Today, the boy with the red shirt took her own bag and left the room. She was very upset and felt not want to talk to anyone.}
\end{itemize}

This example highlights both the strength and the pitfalls of our lookahead steering. On the positive side, HDMI realizes the intended local edit (\textit{his}$\!\rightarrow$\textit{her}) with high \emph{completeness}. In the example, the next sentence produces \textit{She}, indicating that the hidden state has been steered to encode the counterfactual gender value; hence, the \emph{targeted} causal property $Z_c$ (here: pronominal gender at the next relevant decoding step) has been changed completely. However, the subject remains \textit{the boy}, hence, look ahead and back propagating the gender change has failed here, and reduced fluency.

In practice, when the earlier word (subject here) that must adapt is not adjacent to the edit (e.g., a preceding determiner or auxiliary several tokens away), the chain of Jacobians along the expected-embedding path can attenuate, and the gradient signal to that earlier position can effectively vanish. We have observed this with articles and agreement carriers near noun edits: unless $S_{\max}$, step size, and inner steps are tuned, the previous tokens may fail to adapt (or the model may overcompensate elsewhere). This is a promising direction for future work.

\section{Hyperparameters}

For hyperparameter tuning, we performed a grid search, systematically exploring a predefined range of values for each parameter. In the following tables, we provide the fine-tuned parameters for each task.

Probes are linear or 1–hidden–layer MLPs (hidden size selected from $\{64,256,512\}$ by validation accuracy), optimized with AdamW for 100 epochs, weight decay $10^{-6}$, batch size 256.
To avoid leakage, the interventional $Z_c$ probe is trained strictly on the interventional split; $\vZc$/$\vZe$ are trained strictly on the validation–probe split.

\begin{table}[t]
\centering
\footnotesize
\setlength{\tabcolsep}{3pt}
\renewcommand{\arraystretch}{1.05}
\caption{Hyperparameters with the range of values used in the tasks.}
\label{tab:hparams_compact}
\begin{tabularx}{\linewidth}{l X}
\toprule
\textbf{Hyperparameter} & \textbf{Range} \\
\midrule
\texttt{hdmi\_alpha}                  & 1 \\
\texttt{hdmi\_inner\_steps}           & 30 \\
\texttt{alterrep\_alpha}              & \{0.1, 0.5\} \\
\texttt{alterrep\_inlp\_rank\_apply}  & 32 \\
\texttt{probe\_epochs}                & \{75, 100\} \\
\texttt{probe\_lr}                    & $1\times 10^{-2}$ \\
\texttt{probe\_weight\_decay}         & $1\times 10^{-6}$ \\
\texttt{probe\_batch\_size}           & 256 \\
\texttt{probe\_hidden}                & 256 \\
\texttt{inlp\_epochs}                 & \{50, 100\} \\
\texttt{inlp\_lr}                     & $10^{-2}$ \\
\texttt{pgd\_steps}                   & \{40, 50, 100\} \\
\texttt{inlp\_rank}                   & 32 \\
\texttt{gbi\_norm}                    & $\ell_\infty$ \\
\texttt{epsilon}                      & \{0.5, 1, 10\} \\
\bottomrule
\end{tabularx}
\end{table}
\section{System Configuration}

\subsection*{Host and OS}
\begin{itemize}
  \item OS: Ubuntu 22.04.4 LTS
  \item Kernel: Linux 6.8.0-59-generic (x86\_64)
\end{itemize}

\subsection*{Compute}
\begin{itemize}
  \item CPU: AMD EPYC 9454, 2 sockets, 48 cores/socket, 2 threads/core (192 logical)
  \item Memory: 1.5~TiB RAM
  \item GPU: NVIDIA H100 (80~GB HBM3)
\end{itemize}

\subsection*{Storage}
\begin{itemize}
  \item Local NVMe aggregate: \(\sim\)8.6--8.7~TB (XFS)
\end{itemize}

{
\section{More Related Work}

\paragraph{Relation to activation/representation steering.}
Causal mediation analysis aims to measure how a treatment effect is mediated by intermediate variables. \citet{vig2020investigating} uses causal mediation analysis to quantify how specific mediators (neurons/heads) transmit an effect (e.g., gender bias) from input to output; they measure the direct and indirect effects and analyze mechanisms. 
\citet{vig2020investigating} also uses interventions (editing the input or patching the internal representation) to measure causal effects cleanly and to pinpoint which internal components actually cause a behavior. Thus, it is more about using interventions to perform mediation analysis than using mediation analysis to intervene. Our work is the converse: we optimize a logit-margin, which is analogous to an “effect” (but is not the same as the definition of effects in \citet{vig2020investigating}), to intervene on hidden states for causal probing using gradient descent, and we extend this to do local edits in multiple decoding times using the softmax embedding in LA-HDMI.

Our core contribution is to intervene on the hidden state to change the linguistic properties encoded in that hidden state (locally, not steering the model behavior in general), and we evaluate it not with the output sequence of the model, but only with validation probes (section 6 of our paper) on the modified hidden state; we do not perform two forward passes and patch one hidden state into the other to steer the model output. We also do not permanently train/change model parameters as in ROME \citet{meng2022locating}.

\paragraph{Relation to Representation fine‑tuning (ReFT).} ReFT \citep{wu2404reft} is training time representation fine‑tuning that edits hidden representations at chosen layers/ positions. Unlike our method, it updates parameters (or small trainable modules) to change the model’s behavior persistently. So it solves a different problem (persistent model changes for knowledge editing).

\paragraph{Relation to FutureLens.}
FutureLens \citep{pal2023future} is a descriptive method: it asks what future tokens can already be read from a single hidden state, showing that one state can predict tokens several steps ahead. It focuses on describing what is present in the hidden representation, not performing gradient-based adjustments to the hidden state. It does not have contrastive loss like us, and solves a different problem. By contrast, our LA-HDMI uses the differentiable “expected-embedding” path to actively intervene on hidden states toward user‑specified future edit with a margin loss.

\section{Theoretical Analysis}
 Let $\mathcal{V}$ be the vocabulary and $V\triangleq|\mathcal{V}|$.
 For $W\in\mathbb{R}^{V\times D}$, let $w_i^\top$ denote row $i$ (so $w_i\in\mathbb{R}^D$).

\begin{theorem}[Why margin is the right objective for pairwise preference; target-only can fail]
\label{thm:margin_vs_target_l2}
Assume the final-layer logits are affine in the last hidden state $h\in\mathbb{R}^D$:
\[
z(h) = Wh + b \in \mathbb{R}^{V},
\]
where $W\in\mathbb{R}^{V\times D}$ and $b\in\mathbb{R}^{V}$.
Let $w_i^\top$ be row $i$ of $W$, so $z_i(h)=w_i^\top h + b_i$.

Fix two tokens: target $\tau$ and rival/source $\sigma$.
Define the margin
\[
m(h)\triangleq z_\tau(h)-z_\sigma(h) = (w_\tau-w_\sigma)^\top h + (b_\tau-b_\sigma),
\]
and let
\[
d \triangleq w_\tau - w_\sigma \in \mathbb{R}^D.
\]

Consider a norm-bounded intervention $\delta$ with $\|\delta\|_2\le \epsilon$ for some $\epsilon>0$.

Define:
\begin{align*}
\delta_{\mathrm{M}} &\in \arg\max_{\|\delta\|_2\le\epsilon} \; \bigl(z_\tau(h+\delta)-z_\sigma(h+\delta)\bigr)
\quad\text{(margin-optimal)},\\
\delta_{\mathrm{T}} &\in \arg\max_{\|\delta\|_2\le\epsilon} \; z_\tau(h+\delta)
\quad\text{(target-only optimal)}.
\end{align*}

Then:
\begin{enumerate*}
\item \textbf{Margin objective is exactly optimal for pairwise preference.}
If $d\neq 0$, the maximum possible margin improvement equals
$
m(h+\delta_{\mathrm{M}})-m(h) \;=\; \epsilon\|d\|_2,
$
achieved by
$
\delta_{\mathrm{M}}=\epsilon\frac{d}{\|d\|_2}.
$

\item \textbf{Target-only optimizes a different direction and can be arbitrarily worse.}
If $w_\tau\neq 0$, a target-only maximizer is
$
\delta_{\mathrm{T}}=\epsilon\frac{w_\tau}{\|w_\tau\|_2}.
$
The resulting margin improvement is
$
m(h+\delta_{\mathrm{T}})-m(h)
= d^\top \delta_{\mathrm{T}}
= \epsilon\|d\|_2\cos\theta,
$
where $\theta$ is the angle between $d$ and $w_\tau$, i.e.
$
\cos\theta \triangleq \frac{d^\top w_\tau}{\|d\|_2\|w_\tau\|_2}.
$
In particular, it can be \emph{negative} when $\cos\theta<0$.

\item \textbf{Concrete failure case (target-only decreases the margin).}
There exist $W$ such that $m(h+\delta_{\mathrm{T}})-m(h) < 0$ for all $\epsilon>0$.
For example, pick any nonzero $u\in\mathbb{R}^D$ and set
$
w_\tau=u,\qquad w_\sigma=2u.
$
Then $d=w_\tau-w_\sigma=-u$ and the (unique) target-only optimizer is
$\delta_{\mathrm{T}}=\epsilon\,u/\|u\|_2$, which yields
$
m(h+\delta_{\mathrm{T}})-m(h)=(-u)^\top\left(\epsilon\frac{u}{\|u\|_2}\right)=-\epsilon\|u\|_2<0.
$
\end{enumerate*}
\end{theorem}

\begin{proof}
We prove each claim in a direct, linear sequence.

Because $z_i(h)=w_i^\top h + b_i$, we have for any $\delta$:
\[
z_i(h+\delta)=w_i^\top(h+\delta)+b_i=z_i(h)+w_i^\top\delta.
\]

Using the previous identity,
\begin{align*}
m(h+\delta)-m(h)
&=\bigl(z_\tau(h+\delta)-z_\sigma(h+\delta)\bigr)-\bigl(z_\tau(h)-z_\sigma(h)\bigr)\\
&=\bigl(w_\tau^\top\delta - w_\sigma^\top\delta\bigr)\\
&=(w_\tau-w_\sigma)^\top\delta\\
&=d^\top\delta.
\end{align*}
So maximizing the margin after intervention is exactly the same as maximizing $d^\top\delta$
subject to $\|\delta\|_2\le\epsilon$.

By Cauchy--Schwarz, for any feasible $\delta$:
\[
d^\top\delta \le \|d\|_2\|\delta\|_2 \le \|d\|_2\,\epsilon.
\]
So no intervention can improve the margin by more than $\epsilon\|d\|_2$.

This upper bound is achieved by choosing $\delta$ parallel to $d$ with full norm $\epsilon$, i.e.
\[
\delta_{\mathrm{M}}=\epsilon\frac{d}{\|d\|_2}
\quad\Rightarrow\quad
d^\top\delta_{\mathrm{M}}=\epsilon\|d\|_2.
\]
Thus $m(h+\delta_{\mathrm{M}})-m(h)=\epsilon\|d\|_2$, proving Part (1).

\paragraph{The target-only problem and compute its margin effect.}
Maximizing $z_\tau(h+\delta)$ is the same as maximizing $w_\tau^\top\delta$
(since $z_\tau(h)$ is constant w.r.t.\ $\delta$):
\[
\arg\max_{\|\delta\|_2\le\epsilon} z_\tau(h+\delta)
=
\arg\max_{\|\delta\|_2\le\epsilon} w_\tau^\top\delta.
\]
By the same Cauchy--Schwarz argument, a maximizer is
\[
\delta_{\mathrm{T}}=\epsilon\frac{w_\tau}{\|w_\tau\|_2}.
\]
Now plug this $\delta_{\mathrm{T}}$ into the margin change formula from Step 2:
\[
m(h+\delta_{\mathrm{T}})-m(h)=d^\top\delta_{\mathrm{T}}
=\epsilon\frac{d^\top w_\tau}{\|w_\tau\|_2}.
\]
Write $d^\top w_\tau = \|d\|_2\|w_\tau\|_2\cos\theta$ (definition of angle), giving
\[
m(h+\delta_{\mathrm{T}})-m(h)=\epsilon\|d\|_2\cos\theta,
\]
proving Part (2). This shows target-only is optimal for a different linear functional, and its
pairwise preference improvement depends on the alignment between $w_\tau$ and $d$.

\paragraph{A case where target-only decreases the margin.}
Choose any nonzero $u$ and set $w_\tau=u$ and $w_\sigma=2u$.
Then $d=w_\tau-w_\sigma=-u$ and $\delta_{\mathrm{T}}=\epsilon u/\|u\|_2$.
Therefore,
\[
m(h+\delta_{\mathrm{T}})-m(h)=d^\top\delta_{\mathrm{T}}
=(-u)^\top\left(\epsilon\frac{u}{\|u\|_2}\right)=-\epsilon\|u\|_2<0,
\]
proving Part (3).
\end{proof}

}

\section{Use of Large Language Models}
We used Large Language Models (LLMs) to aid or polish the manuscript text. Specifically, LLMs were used to improve grammar, phrasing, and clarity of exposition; they were also used for code debugging.

\end{document}